\documentclass{article}
\usepackage{paperstyle,times}

\usepackage{amsmath,amsfonts,bm}

\def\eqref#1{equation~\ref{#1}}

\def\1{\bm{1}}

\DeclareMathAlphabet{\mathsfit}{\encodingdefault}{\sfdefault}{m}{sl}
\SetMathAlphabet{\mathsfit}{bold}{\encodingdefault}{\sfdefault}{bx}{n}

\usepackage{hyperref}
\usepackage{url}
\usepackage{amsmath}
\usepackage{amssymb}
\usepackage{amsthm}
\newtheorem{theorem}{Theorem}
\usepackage{caption}
\usepackage{float}
\usepackage{makecell}
\usepackage{graphicx}
\usepackage{wrapfig}
\usepackage{subfigure}
\usepackage{booktabs}
\usepackage{multirow}
\usepackage{algorithm}
\usepackage{algpseudocode}

\title{The Error You See Is Not the Error You Made: Progression-aware Reasoning Origin for Reasoning Error Localization}

\author{\textbf{Yiguo Wang}$^{a,b}$, \textbf{Ziyuan Yang}$^{b}$, \textbf{Yi Zou}$^{c}$, \textbf{Dan Lin}$^{d}$, \textbf{Rongsheng Li}$^{d}$, \textbf{Yi Zhang}$^{b}$\\[0.3em]
\normalfont
$^{a}$ Xiaopeng Honors College, Nanchang Hangkong University, China\\
$^{b}$ School of Cyber Science and Engineering, Sichuan University, China\\
$^{c}$ School of Software, Nanchang Hangkong University, China\\
$^{d}$ School of Computer Science and Technology, Harbin Engineering University, China\\
}
\date{}

\finalcopy

\begin{document}

\maketitle

\begin{abstract}
Verifying multi-step LLM reasoning requires more than determining whether a trace is correct: a useful verifier should identify where the reasoning first goes wrong. However, existing holistic methods provide little positional evidence, while forward sequential verification often treats the first rejected step as the error source. Under error propagation, this assumption can fail, since an earlier mistake may remain locally plausible and become observable only through its downstream consequences. We therefore \emph{rethink reasoning verification as a progression-aware error-source localization problem}: rather than asking only where a reasoning trace first appears inconsistent, we ask which earlier step best explains how that inconsistency emerges along the trajectory. Based on this view, we propose \textbf{Progression-aware Reasoning Origin (PRO)}, a training-free framework for first-error localization. PRO jointly models incoming support from the preceding context and outgoing compatibility with subsequent reasoning, selectively refines regions where these signals disagree, and finally performs detector-conditioned source attribution with intervention-based evidence to distinguish the true error origin from its propagated manifestations. We further formalize the gap between forward rejection and structural exposure, showing why incoming-side evidence alone is insufficient for reliable localization under error propagation. Experiments across open-form, medical, and structured reasoning tasks demonstrate consistent improvements over strong verification baselines, supporting progression-aware source attribution as a more faithful formulation of reasoning verification.

\end{abstract}

\section{Introduction}

Large language models (LLMs) are increasingly capable of solving complex multi-step reasoning tasks, driven in part by advances in chain-of-thought prompting and zero-shot reasoning \citep{wei2022chain,kojima2022large}, decoding-time aggregation such as self-consistency \citep{wang2023selfconsistency}, and more structured inference-time search paradigms \citep{yao2023tree,besta2023graph}. Yet verifying their reasoning remains difficult. Final-answer checking is insufficient because a trace can contain invalid intermediate steps even when it reaches a correct answer. A useful verifier must therefore determine not only whether a trace is flawed, but also which step first introduces the error.

Existing verification methods mainly follow two paradigms. Related answer-level self-verification and critique frameworks further show the practical value of test-time checking, but they still focus primarily on improving or validating complete responses rather than localizing the earliest corrupted step \citep{weng2023selfverification,gou2024critic}. Holistic verifiers assess the complete chain and return a single judgment, providing little positional evidence for error attribution. Sequential verifiers instead decompose a trace into smaller units and validate them step by step, improving interpretability and control. However, standard sequential verification is driven primarily by forward local checking: it asks whether the current block is supported by the problem and the accepted prefix. Under error propagation, the first block rejected by this procedure can be a downstream exposure point rather than the step that introduced the error. A locally plausible block may become suspicious only when its consequences appear in later reasoning, creating a gap between error exposure and error attribution.

\textbf{\emph{When errors propagate through a reasoning trace, how can a verifier localize the first erroneous step rather than a later step at which the error becomes observable?}}

Addressing this question requires moving beyond verification as a sequence of independent accept-or-reject decisions. Reasoning errors are inherently \emph{progressive}: once introduced, they can distort subsequent context while leaving later steps locally coherent with a corrupted trajectory. Consequently, the first observable failure may occur after the true error source, creating a gap between error origin and error exposure. Reliable localization therefore requires reasoning not only about local validity, but also about how verification evidence evolves along the trajectory. Downstream consequences can provide retrospective evidence that exposes earlier mistakes, while later rejected steps may merely reflect propagated errors. This motivates explicitly modeling the progression of verification evidence to distinguish where an error originates from where it first becomes observable.

Motivated by the progression-aware view, we propose \textbf{Progression-aware Reasoning Origin~(PRO)}, a training-free framework for progression-aware error-source localization in decomposed reasoning traces. PRO treats each reasoning step as both a conclusion supported by its preceding context and a premise that should remain compatible with its subsequent continuation. Accordingly, it evaluates complementary \emph{incoming} and \emph{outgoing} evidence: the former measures whether a step is justified by the problem and the verified prefix, while the latter examines whether the step can consistently support the reasoning that follows. Their agreement provides stronger evidence for localization, whereas disagreement indicates a structurally ambiguous region in which the exposed failure may not coincide with the true source.

Rather than making an immediate decision at such ambiguous positions, PRO progressively narrows the candidate source. It first performs coarse-grained trajectory-level verification to identify suspicious regions, then selectively refines the reasoning granularity only where the evidence remains inconclusive. Finally, conditioned on the observed verification pattern, PRO compares the remaining source candidates to distinguish the step that introduces the error from downstream steps that merely inherit or expose it. This coarse-to-fine procedure enables PRO to exploit both local consistency and downstream consequences while avoiding the forward-only assumption that the earliest rejected step is necessarily the error origin. Our main contributions can be summarized as:
\begin{itemize}
    \item We formulate reasoning verification as a first-error localization problem, and argue that accurate localization requires more than purely forward local checking.
     \item We propose PRO, a progression-aware verification framework that combines backward consistency checking, mismatch-triggered fine-grained refinement, and top-\(k\) candidate localization.
    \item We formalize the discrepancy between forward rejection and structural exposure under error propagation, providing a theoretical basis for incorporating outgoing evidence in earliest-error localization.
\end{itemize}

\section{Related Work}
\textbf{\textit{Reasoning Verification.}} Existing verification methods are mainly holistic or decomposition-based. Early verifier work learns critics to score complete sampled solutions and rerank candidates \citep{cobbe2021training,uesato2022solving}. Recent methods move toward finer process assessment by inspecting steps, blocks, or explicit structural relations. GoV organizes verification units into graph-structured blocks with adaptive granularity \citep{fang2025graphverification}, while ProcessBench and DeltaBench provide earliest-error benchmarks and show that process-error detection remains difficult even for strong PRMs and critic-style verifiers \citep{zheng2025processbench,he2025deltabench}.

\textbf{\textit{Process Supervision.}} Process supervision provides step-level feedback rather than relying only on final outcomes. \citet{uesato2022solving} compare process- and outcome-based supervision for mathematical reasoning, and \citet{lightman2023lets} establish the modern PRM paradigm with PRM800K. Subsequent work reduces annotation cost through automated process labeling, as in Math-Shepherd \citep{wang2024mathshepherd} and OmegaPRM \citep{luo2024improve}. Together, these methods make dense intermediate supervision a standard tool for verifier training, reranking, and reinforcement learning.

\textbf{\textit{Self-Correction and Test-Time Verification.}} Another related line improves reasoning by critique, revision, and verification at inference time. Representative examples here include Self-Refine \citep{madaan2023selfrefine}, Chain-of-Verification \citep{dhuliawala2024chain}, ProCo \citep{wu2024proco}, and Derailer-Rerailer \citep{wan2025derailerrerailer}. At the same time, recent analysis also shows that intrinsic self-correction remains limited without strong external feedback or explicit verification signals \citep{huang2023large}. These methods mainly focus on improving or stabilizing generated reasoning, rather than locating the earliest erroneous step under explicit process annotations.

\section{Detector-Stage Gap in Sequential Verification}
\subsection{Incoming Admissibility and Structural Exposure}
The limitation of forward-only verification arises from the evidence available at each block. A standard sequential verifier evaluates \(b_i\) only from its incoming context, namely the problem \(q\) and the accepted prefix \(\mathcal{P}_{i-1}\). Consequently, a block may remain locally admissible even when it introduces a defect whose inconsistency becomes visible only through the subsequent transition. In such cases, the earliest forward rejection occurs after the underlying error has already entered the trajectory. To characterize this gap, we distinguish incoming admissibility from outgoing compatibility.

Specifically, we define $I_i=\mathbf{1}\!\left[(q,\mathcal{P}_{i-1}) \models b_i\right],
O_i=\mathbf{1}\!\left[b_i \vdash b_{i+1}\right], i<n$, where $\models$ denotes semantic-logical support from the problem and preceding reasoning, and $b_i \vdash b_{i+1}$ denotes whether $b_i$ provides a valid basis for the next reasoning block. Thus, $I_i$ characterizes whether $b_i$ is admissible from its incoming context, whereas $O_i$ characterizes whether its local continuation remains structurally compatible with the subsequent reasoning.

We then define the mismatch set $\mathcal{M}(q,\mathcal{B})
=
\{\,i<n \mid I_i=1,\ O_i=0\,\}$, which contains blocks that appear admissible under forward checking but fail to support their immediate continuation. These blocks are precisely the cases that cannot be exposed from prefix evidence alone.

Under an ideal forward-only verifier, the earliest rejection index is
$\tau_{\mathrm{seq}}(q,\mathcal{B})
=
\min\bigl(\{\,i \mid I_i=0\,\}\cup\{\infty\}\bigr)$. By contrast, once outgoing compatibility is also considered, the earliest detector-stage exposure becomes $\tau_{\mathrm{exp}}(q,\mathcal{B})
=
\min\bigl(
\{\,i \mid I_i=0\,\}
\cup
\mathcal{M}(q,\mathcal{B})
\cup
\{\infty\}
\bigr)$.

Importantly, \(\tau_{\mathrm{exp}}\) is not yet the final first-error localization \(f(q,\mathcal{B})\). It only identifies the earliest block at which either an incoming inconsistency or an outgoing incompatibility becomes detectable. The exposed block may itself be erroneous, or it may reflect the downstream consequence of an earlier error. A subsequent source-attribution stage is therefore required to determine which preceding block first introduced the inconsistency.

\subsection{Formal Properties of the Detector-Stage Gap}
The definitions above nevertheless reveal a fundamental limitation of forward-only detection. If outgoing compatibility provides no additional information, then $\tau_{\mathrm{exp}} $reduces to the ordinary forward-rejection index $\tau_{\mathrm{seq}}$. However, whenever a block remains admissible from its prefix but fails to support the subsequent transition, the outgoing-side test can expose a suspicious location before any forward rejection occurs. This suggests two questions: whether such earlier exposure is formally guaranteed to be no later than forward rejection, and whether it can be recovered from incoming-side evidence alone. The following results answer these questions.
\begin{theorem}[Strict Gap Between Forward Rejection and Structural Exposure]
To make the comparison with ordinary forward rejection explicit, the ordering relation can be formulated as:
\begin{equation}
\tau_{\star}(q,\mathcal{B})\le \tau_{\mathrm{seq}}(q,\mathcal{B}).
\end{equation}
Moreover, there exist \(q\) and \(\mathcal{B}\) such that the inequality is strict.
\end{theorem}

\begin{theorem}[Necessity of Outgoing-Side Information]
To show why incoming-side evidence alone cannot determine structural exposure, we consider two traces that are indistinguishable from prefix-only signals. Their separation can be expressed as:
\begin{equation}
\forall k,\ I_k(\mathcal{B})=I_k(\mathcal{B}'), \qquad
\tau_{\star}(q,\mathcal{B})\neq \tau_{\star}(q,\mathcal{B}').
\end{equation}
Consequently, no detector rule that depends only on incoming-side signals \(\{I_k\}\) can recover \(\tau_{\star}\) correctly for all decomposed reasoning traces.
\end{theorem}

These results do not by themselves recover the final source target \(f(q,\mathcal{B})\). Instead, they motivate a forward-anchored detector augmented with an outgoing-side structural test. PRO subsequently refines these pre-refinement exposure candidates before using the surviving detector output for source localization.

\section{Progression-aware Reasoning Origin}
\label{sec:method}

\subsection{Overview}
In this paper, we revisit reasoning verification from a \emph{progression-aware} perspective. Beyond checking whether each block is supported by its preceding context, we also examine its compatibility with subsequent reasoning, allowing downstream inconsistency to provide retrospective evidence for distinguishing error origins from propagated failures. Motivated by this, we propose a training-free framework PRO for first-error localization. As illustrated in Figure~\ref{fig:pro_method}, PRO consists of three stages. First, \textbf{progression-aware evidence modeling} jointly examines the incoming support of each block and its compatibility with subsequent reasoning to identify potential progression inconsistencies. Second, \textbf{mismatch-triggered granularity refinement} selectively re-examines ambiguous regions at a finer reasoning granularity, reducing uncertainty introduced by coarse block decomposition and local transition noise. Third, \textbf{origin-oriented candidate attribution} uses the refined exposure position as a detector anchor, constructs a detector-conditioned shortlist, and further compares these candidates with rewrite-supported evidence, thereby tracing the observed failure back to its most plausible error origin. Through this progression from inconsistency detection to ambiguity resolution and source attribution, PRO explicitly separates \emph{where an error becomes observable} from \emph{where it is introduced}.

\begin{figure}[t]
    \centering
    \includegraphics[width=\linewidth]{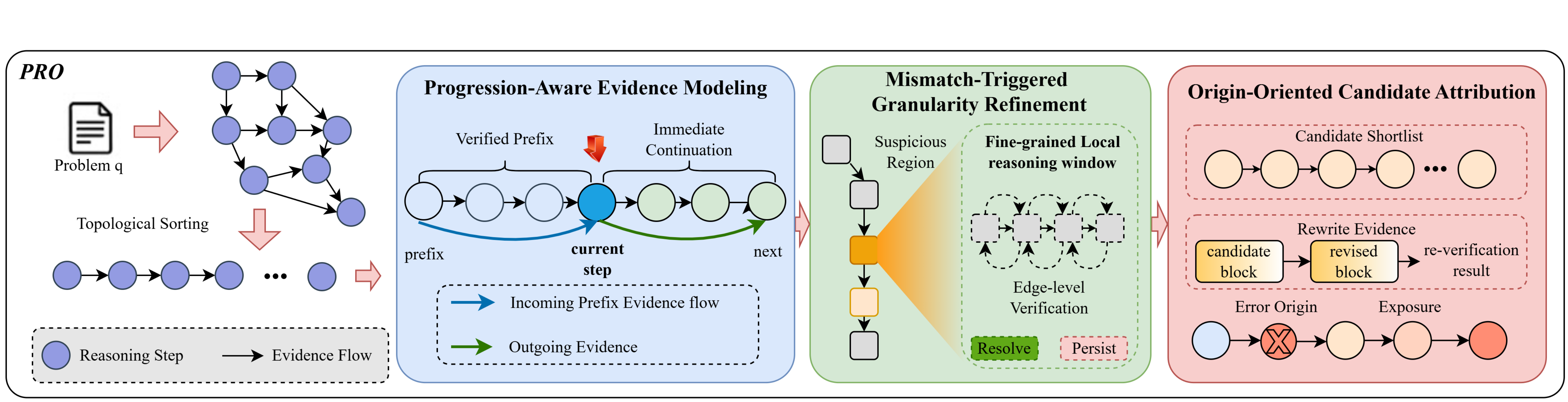}
        \caption{The overview of the proposed PRO.}
        \vspace{-15pt}
    \label{fig:pro_method}
\end{figure}

\subsection{Progression-aware Evidence Modeling}
\label{sec:evidence}
As discussed above, conventional sequential verification relies primarily on forward local checking, and therefore tends to identify the earliest block at which an error becomes observable rather than the block that actually introduces it. Under error propagation, these two positions may differ: an earlier erroneous block can remain locally consistent with the accepted prefix, while its consequence is only exposed in subsequent reasoning.

To bridge this gap, PRO adopts a \emph{progression-aware} view of verification. Rather than judging each block solely from its preceding context, PRO considers its role within the local reasoning progression. For a block \(b_i\), evidence can be examined in two directions. From the preceding context to \(b_i\), we ask whether the current block is a valid conclusion supported by the problem and the verified prefix; we refer to this as \emph{incoming evidence}. From \(b_i\) to the subsequent reasoning, we ask whether the current block can serve as a valid premise for the continuation; we refer to this as \emph{outgoing evidence}. Together, the two signals characterize whether a block is locally justified while remaining consistent with the reasoning progression.

\textbf{Incoming evidence.}
Let \(\mathcal{P}_{i-1}\) denote the verified prefix before \(b_i\). We define
$V_{\mathrm{in}}(b_i)
=
\mathrm{Verify}_{\mathrm{in}}
\bigl(
b_i;q,\mathcal{P}_{i-1}
\bigr)$,
where \(V_{\mathrm{in}}(b_i)\in\{0,1\}\) indicates whether \(b_i\) is justified by the problem and its accepted prefix. This corresponds to the ordinary forward-anchored verification signal.

\textbf{Outgoing evidence.}
For \(i<n\), PRO further evaluates whether the same block can serve as a valid premise for the immediate continuation:
\begin{equation}
V_{\mathrm{out}}(b_i)
=
\mathrm{Verify}_{\mathrm{out}}
\bigl(
b_i;q,b_{i+1}
\bigr).
\end{equation}
Rather than reversing the full reasoning chain, this test examines whether the local progression from \(b_i\) to \(b_{i+1}\) remains coherent. For the terminal block, we set \(V_{\mathrm{out}}(b_n)=1\).

The incoming and outgoing signals jointly define a progression-aware verification state for each block. Specifically, we introduce:
\begin{equation}
s_i
=
\begin{cases}
\mathrm{R}, & V_{\mathrm{in}}(b_i)=0,\\[2pt]
\mathrm{A}, & V_{\mathrm{in}}(b_i)=1 \ \wedge\ V_{\mathrm{out}}(b_i)=1,\\[2pt]
\mathrm{M}, & V_{\mathrm{in}}(b_i)=1 \ \wedge\ V_{\mathrm{out}}(b_i)=0,
\end{cases}
\end{equation}
where \(\mathrm{R}\), \(\mathrm{A}\), and \(\mathrm{M}\) denote \emph{rejection}, \emph{agreement}, and \emph{mismatch}, respectively. The corresponding index sets are $\mathcal{R}=\{i:s_i=\mathrm{R}\}$,
$\mathcal{A}=\{i:s_i=\mathrm{A}\}$, and
$\mathcal{M}=\{i:s_i=\mathrm{M}\}$, respectively.

This partition reflects the asymmetric roles of the two verification signals. A block enters \(\mathcal{R}\) as soon as its incoming evidence fails, indicating that the inconsistency is already exposed under forward verification. When the incoming evidence is positive, the outgoing signal further distinguishes between progression agreement and progression mismatch. Thus, \(\mathcal{A}\) contains blocks that are supported by both their preceding context and subsequent progression, whereas
$\mathcal{M}
=
\left\{
i:
V_{\mathrm{in}}(b_i)=1,\,
V_{\mathrm{out}}(b_i)=0
\right\}$ captures the cases in which forward verification alone remains insufficient. We refer to these positions as \emph{progression mismatches}.

Among the three verification states, \(\mathcal{R}\) and \(\mathcal{A}\) admit relatively direct interpretations: blocks in \(\mathcal{R}\) are already rejected by the incoming verification, whereas blocks in \(\mathcal{A}\) remain consistent with both their preceding context and subsequent progression. The ambiguity therefore concentrates in \(\mathcal{M}\). For \(i\in\mathcal{M}\), the incoming evidence accepts \(b_i\), while the outgoing evidence indicates that the progression from \(b_i\) to \(b_{i+1}\) is inconsistent. This disagreement localizes the uncertainty to the neighborhood of the transition, but does not yet reveal whether the underlying error lies in \(b_i\), in \(b_{i+1}\), or in a finer reasoning step hidden by the current block decomposition. PRO therefore treats \(\mathcal{M}\) as the set of unresolved progression regions and subjects only these regions to the granularity-refinement stage described next.

\subsection{Mismatch-triggered Granularity Refinement}
\label{sec:refinement}

The mismatch set \(\mathcal{M}\) identifies where forward-local verification becomes insufficient, but it still does not reveal which fine-grained inference actually breaks the progression. In natural-language reasoning, a single block often compresses multiple semantic or computational operations, so the true point of failure may be hidden inside the block or at the transition into the next one. As a result, a progression mismatch should not yet be treated as a final detector decision. Instead, it marks a structurally ambiguous region that requires more precise inspection before source attribution.

To resolve this ambiguity, PRO introduces a mismatch-triggered coarse-to-fine refinement stage. Crucially, this additional verification is not applied to the entire trace. It is activated only for \(i\in\mathcal{M}\), so that the extra effort is concentrated on local regions where the progression signal remains unresolved.

\textbf{Local progression window.} For each \(i\in\mathcal{M}\), we define the local progression window as $\mathcal{W}_i=(b_i,b_{i+1})$, which isolates the neighborhood in which the disagreement first appears. We then decompose this local window into finer reasoning units:
\begin{equation}
\widehat{\mathcal{W}}_i
=
\mathrm{Refine}(\mathcal{W}_i)
=
(u_{i,1},u_{i,2},\ldots,u_{i,m_i}),
\qquad m_i\ge 2.
\end{equation}

\textbf{Fine-grained re-evaluation.}
Once the ambiguous region has been exposed at a finer resolution, PRO re-evaluates whether the local progression still fails:
\begin{equation}
\widehat{V}_i
=
\mathrm{Verify}_{\mathrm{fine}}
\bigl(
\widehat{\mathcal{W}}_i;q
\bigr)
\in\{0,1\}.
\end{equation}
Here, \(\widehat{V}_i=1\) means that the coarse mismatch can be reconciled after refinement, whereas \(\widehat{V}_i=0\) means that the inconsistency persists even under finer inspection. This step therefore determines whether the original disagreement reflects only a coarse decomposition artifact or a more robust local failure.

\textbf{Refined detector state.}
These refined outcomes are then folded back into the detector stage. Specifically, the detector state is defined as: \(D_i=1\) for \(i\in\mathcal{R}\) and for \(i\in\mathcal{M}\) with \(\widehat{V}_i=0\), whereas \(D_i=0\) for \(i\in\mathcal{A}\) and for \(i\in\mathcal{M}\) with \(\widehat{V}_i=1\).

This definition preserves the direct interpretations of \(\mathcal{R}\) and \(\mathcal{A}\), while using refinement only to adjudicate the ambiguous mismatch cases. The earliest position at which an inconsistency survives refinement is:

\begin{equation}
\tau_{\mathrm{det}}(q,\mathcal{B})
=
\min
\left(
\{i\mid D_i=1\}\cup\{\infty\}
\right).
\end{equation}

In this way, refinement converts a coarse progression mismatch into a more reliable detector signal. We emphasize that \(\tau_{\mathrm{det}}\) is still an \emph{exposure-oriented} quantity: it marks the earliest robust point at which the trajectory can no longer be locally reconciled after ambiguity has been reduced, rather than directly asserting the first erroneous block. This refined exposure signal is therefore what the next stage uses as an anchor for source attribution.

\subsection{Origin-oriented Candidate Attribution}
\label{sec:localization}

The refined detector output identifies where the failure becomes robustly observable, but observable failure is still not the same as causal origin. A downstream block may be the first point that remains irreconcilable after refinement, even though the actual error was introduced earlier and only propagated to that position. The role of the final stage is therefore to reason backward from this refined exposure signal and determine which earlier block best explains the observed failure.

\textbf{Candidate shortlisting.}
Let the refined detector return
\(
j_{\mathrm{det}}\in\{1,\ldots,n\}
\);
if no inconsistency survives refinement, PRO predicts that no localizable error has been detected. Otherwise, \(j_{\mathrm{det}}\) provides an anchor for source localization. In principle, the error origin should not occur later than the point where the failure becomes observable. In the implemented pipeline, we therefore use \(j_{\mathrm{det}}\) to condition candidate attribution while still allowing the proposal stage to score the full trace, so that the detector output is treated as a strong exposure-aware hypothesis rather than a hard truncation rule.

Rather than treating every proposed block equally, PRO first obtains an ordered proposal list over the full trace as:

\begin{equation}
\pi
\bigl(
q,\mathcal{B}
\bigr)
=
(c_1,\ldots,c_m),
\qquad
m\le k,
\end{equation}
Here, \(\pi(q,\mathcal{B})\) denotes the ordered output of an LLM-based source proposal function over the full trace. The detector position is then retained in the shortlist so that the exposed failure itself remains available as a competing hypothesis even when the proposal model would otherwise omit it. We write the resulting candidate set as:

\begin{equation}
\pi_k^{\star}
\bigl(
q,\mathcal{B};j_{\mathrm{det}}
\bigr)
=
(c_1^{\star},\ldots,c_{m^{\star}}^{\star}),
\qquad
m^{\star}\le k.
\end{equation}
This shortlisting step narrows the attribution problem from the full trajectory to a small set of plausible source hypotheses, while keeping the refined detector position available as a fallback explanation.

\textbf{Intervention-based source evidence.}
Observation alone may still favor a downstream manifestation, particularly when several preceding blocks remain locally plausible under the corrupted trajectory. To distinguish these competing hypotheses, PRO introduces a lightweight intervention for each shortlisted candidate. For \(c_r^{\star}\), the corresponding block is locally repaired under the same preceding context:
\begin{equation}
\widetilde{b}_{c_r^{\star}}
=
\mathrm{Rewrite}
\bigl(
q,\mathcal{P}_{c_r^{\star}-1},
b_{c_r^{\star}}
\bigr).
\end{equation}
The repaired block is then re-evaluated using the same incoming and outgoing evidence. Let
\(
\widetilde{V}_{\mathrm{in}}^{(r)}
\)
and
\(
\widetilde{V}_{\mathrm{out}}^{(r)}
\)
denote the resulting verification signals. We define:
\begin{equation}
S_r
=
\mathbb{I}
\left[
\mathrm{Fuse}
\left(
\widetilde{V}_{\mathrm{in}}^{(r)},
\widetilde{V}_{\mathrm{out}}^{(r)}
\right)
=
\mathrm{CORRECT}
\right],
\end{equation}
with an optional local refinement when the fused verdict remains ambiguous.

The purpose of \(S_r\) is not to independently certify a candidate as the first error. Instead, it measures whether repairing that candidate removes the local inconsistency associated with the observed progression failure. A candidate whose repair restores consistency receives stronger evidence of lying upstream of the exposure than one whose repair leaves the failure unchanged. In this way, rewrite-based verification functions as source-oriented intervention evidence rather than as a standalone correction heuristic.

\textbf{Source attribution.}
Once the detector anchor and intervention-based evidence have been assembled, PRO makes the final attribution decision by jointly considering candidate chronology, detector position, trajectory context, and intervention outcomes:

\begin{equation}
j_{\mathrm{src}}
=
\mathrm{Rerank}
\left(
q,
\mathcal{B},
j_{\mathrm{det}},
\left\{
(c_r^{\star},r,S_r)
\right\}_{r=1}^{m^{\star}}
\right).
\end{equation}
The localized first-error prediction can be formulated as:
\begin{equation}
f(q,\mathcal{B})
=
\begin{cases}
j_{\mathrm{src}},
&
j_{\mathrm{src}}
\in
\{c_1^{\star},\ldots,c_{m^{\star}}^{\star}\},
\\
j_{\mathrm{det}},
&
\text{otherwise}.
\end{cases}
\end{equation}

This final stage completes the progression-aware pipeline introduced above. PRO first identifies where the trajectory becomes robustly inconsistent, then uses that exposure signal to organize candidate hypotheses, test them through intervention, and trace the failure back to its most plausible source. By explicitly separating exposure detection from source attribution, PRO can use downstream consequences as retrospective evidence without assuming that the first rejected block is necessarily the first erroneous one.

\section{Experiments}
\subsection{Experimental Setting}
We compare PRO with three baselines: Holistic Verification, GoV \citep{fang2025graphverification}, and PARC \citep{mukherjee2025parc}. Holistic Verification evaluates the full reasoning trace in a single pass, GoV represents forward-only decomposition-based verification, and PARC provides a premise-augmented graph baseline. All experiments use deterministic decoding at temperature \(0\), and PRO uses a detector-conditioned shortlist of \(k=5\). For ProcessBench, Holistic Verification follows the official leaderboard result, while the other methods are evaluated under matched backbone settings.

We evaluate PRO on three tasks with complementary reasoning structures. For open-form reasoning, we use the GSM8K subset of ProcessBench~\citep{zheng2025processbench}, containing 400 traces with human-annotated first-error positions. For medical reasoning, we construct a 450-trace MedReason benchmark from publicly available traces~\citep{wu2025medreason} following the MedPRMBench corruption protocol~\citep{wu2026medprmbench}. For structured arithmetic reasoning, we use Number Triangle Summation with depths ($N\in\{2,4,6,8\}$) and a 50\% corruption rate. Metric are detailed in Appendix~\ref{sec:appendix-metrics}.

\subsection{Open-Form Reasoning Experiment}

\begin{wraptable}{r}{0.5\columnwidth}
\vspace{-1.5em}
\centering
\tiny
\setlength{\tabcolsep}{1.1pt}
\renewcommand{\arraystretch}{0.84}
\caption{ProcessBench. Performance comparison on ProcessBench (\%).}
\vspace{-10pt}
\label{tab:processbench_results}
\begin{tabular}{@{}clcccc@{}}
\toprule
\makecell[c]{Method} & \makecell[c]{Metric} & \makecell[c]{Qwen2.5-32B\\-Instruct} & \makecell[c]{Qwen2.5-72B\\-Instruct} & \makecell[c]{DeepSeek\\-V3.2} & \makecell[c]{Llama3.3-70B\\-Instruct} \\
\midrule
\multirow{2}{*}{\shortstack{Holistic\\Verification}} & Correct Acc & 98.45 & 93.26 & 79.27 & 94.30 \\
& Error Acc & 45.89 & 65.70 & 80.68 & 68.60 \\
\midrule
\multirow{2}{*}{PARC} & Correct Acc & 83.94 & 75.13 & 82.90 & 89.12 \\
& Error Acc & 57.97 & 68.60 & 77.78 & 63.77 \\
\midrule
\multirow{2}{*}{GoV} & Correct Acc & 94.82 & 97.41 & 95.34 & 93.26 \\
& Error Acc & 68.12 & 71.01 & 79.71 & 68.60 \\
\midrule
\multirow{2}{*}{PRO} & Correct Acc & 94.30 & 90.67 & 95.34 & 92.75 \\
& Error Acc & 68.12 & 76.81 & 80.68 & 70.53 \\
\bottomrule
\end{tabular}
\vspace{-1.6em}
\end{wraptable}

We first evaluate PRO under open-form reasoning using the GSM8K subset of ProcessBench. This setting contains naturally written multi-step solutions whose intermediate steps are not constrained by a fixed reasoning template, making the correspondence between the first erroneous step and the first observable failure less explicit. This experiment is designed to assess whether PRO can reliably localize the error source when reasoning errors may propagate through semantically plausible intermediate steps before becoming observable.

\begin{wrapfigure}{r}{0.4\columnwidth}
\vspace{-1.2em}
\centering
\includegraphics[width=0.42\columnwidth]{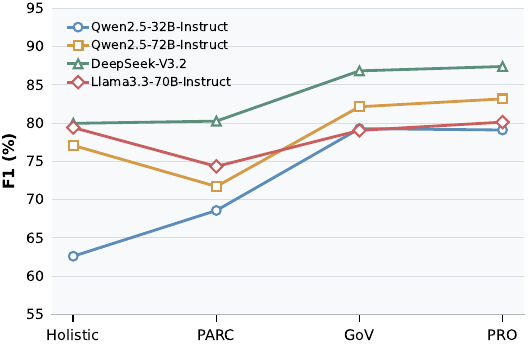}
\vspace{-15pt}
\caption{F1 trends across methods. Each line denotes one backbone.}
\label{fig:processbench-main-results}
\vspace{-15pt}
\end{wrapfigure}
Table~\ref{tab:processbench_results} reports the quantitative comparison across different verification methods and backbone models. PRO generally achieves stronger first-error localization while maintaining competitive performance on correct traces, indicating that progression-aware verification provides a more reliable basis for source attribution than holistic or purely forward-local verification. 
Figure~\ref{fig:processbench-main-results} further illustrates the overall performance trend across backbones. The visual comparison shows that the advantage of PRO is not tied to a particular model, but persists across different backbone choices. This consistency suggests that the improvement mainly comes from the verification strategy itself by explicitly separating error exposure from error origin, PRO is better able to recover the source of failures in loosely structured reasoning traces.

\subsection{Medical Reasoning Experiment}
\begin{wraptable}{r}{0.54\columnwidth}
\vspace{-1.2em}
\centering
\tiny
\setlength{\tabcolsep}{1.1pt}
\renewcommand{\arraystretch}{0.84}
\caption{MedReason. Performance comparison on MedReason (\%).}
\label{tab:medreason_results}
\begin{tabular}{@{}clcccc@{}}
\toprule
\makecell[c]{Method} & \makecell[c]{Metric} & \makecell[c]{Qwen2.5-32B\\-Instruct} & \makecell[c]{Qwen2.5-72B\\-Instruct} & \makecell[c]{DeepSeek\\-V3.2} & \makecell[c]{Llama3.3-70B\\-Instruct} \\
\midrule
\multirow{2}{*}{\shortstack{Holistic\\Verification}} & Correct Acc & 98.67 & 84.89 & 33.33 & 95.56 \\
& Error Acc & 17.78 & 27.11 & 44.89 & 11.56 \\
\midrule
\multirow{2}{*}{PARC} & Correct Acc & 88.00 & 78.22 & 68.44 & 89.33 \\
& Error Acc & 25.33 & 35.56 & 23.56 & 17.78 \\
\midrule
\multirow{2}{*}{GoV} & Correct Acc & 81.78 & 80.44 & 59.11 & 88.44 \\
& Error Acc & 27.56 & 28.89 & 33.78 & 12.44 \\
\midrule
\multirow{2}{*}{PRO} & Correct Acc & 71.11 & 72.44 & 48.89 & 86.67 \\
& Error Acc & 40.00 & 38.22 & 42.22 & 27.11 \\
\bottomrule
\end{tabular}
\vspace{-1.4em}
\end{wraptable}
We evaluate PRO on medical reasoning to examine whether its source-localization capability extends beyond general-purpose reasoning to a domain with stronger semantic and clinical constraints. Compared with open-form arithmetic reasoning, medical traces introduce an additional challenge: intermediate statements can remain medically plausible in isolation even when an earlier decision has already placed the reasoning on an incorrect diagnostic path. As a result, the first locally suspicious step may be separated even further from the decision that actually introduces the error. The MedReason setting therefore tests whether PRO can preserve accurate first-error localization when error propagation occurs within domain-consistent but clinically consequential reasoning trajectories.

\begin{wrapfigure}{r}{0.42\columnwidth}
\vspace{-1.5em}
\centering
\includegraphics[width=0.42\columnwidth]{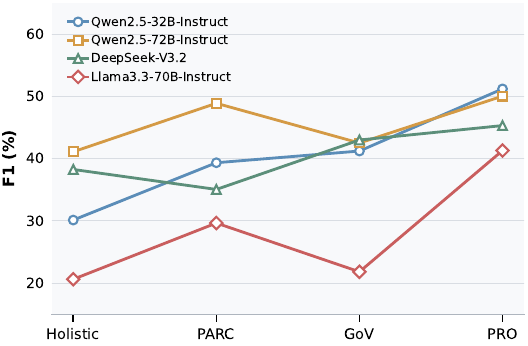}
\vspace{-15pt}
\caption{MedReason F1 trends. Each line denotes one backbone.}
\label{fig:medreason-f1-trends}
\vspace{-15pt}
\end{wrapfigure}

Table~\ref{tab:medreason_results} reports the quantitative comparison on MedReason. PRO shows consistently stronger first-error localization across the evaluated backbones, with the improvement primarily reflected in its ability to identify erroneous traces while retaining reliable recognition of correct reasoning. This suggests that progression-aware verification is beneficial in medical reasoning, where locally plausible statements can obscure an earlier diagnostic error and make forward-local judgments less reliable for source attribution. 

Figure~\ref{fig:medreason-f1-trends} further shows the overall performance trend across different backbones. PRO maintains a clear and stable advantage despite changes in the underlying model, indicating that the improvement is not specific to a particular backbone. This trend supports the central premise of PRO, under domain-constrained reasoning, downstream clinical coherence does not necessarily identify the true error source, and explicitly reasoning over error progression helps recover the earlier step that initiates the failure.

\subsection{Cross-Backbone Statistics}

To assess whether the improvement of PRO is robust to the choice of backbone, we further aggregate its performance gains across the four evaluated models. Specifically, for each benchmark and baseline, we compute the F1 improvement of PRO under each backbone and then average these gains across backbones.

\begin{wrapfigure}{r}{0.46\textwidth}
\vspace{-1em}
\centering
\includegraphics[width=0.44\textwidth]{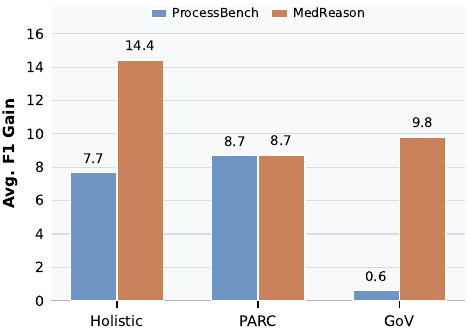}
\vspace{-10pt}
\caption{Cross-backbone summary statistics.}
\label{fig:main_result_gain_summary}
\vspace{-2em}
\end{wrapfigure}
Figure~\ref{fig:main_result_gain_summary} summarizes the resulting average gains on ProcessBench and MedReason. PRO maintains positive improvements over all baselines across both benchmarks, indicating that its advantage is not driven by a particular backbone or isolated experimental configuration. The improvement is especially more pronounced on MedReason, where clinically plausible intermediate reasoning makes the separation between error exposure and error origin more challenging. 

\subsection{Structured Arithmetic Reasoning}

\begin{wraptable}{l}{0.44\textwidth}
\vspace{-1.5em}
\captionof{table}{The results on Number Triangle Summation (\%).}
\vspace{-10pt}
\centering
\scriptsize
\setlength{\tabcolsep}{5pt}
\begin{tabular}{lcc}
\toprule
\multirow{2}{*}{Problem Size} & \multicolumn{2}{c}{Llama3.3-70B-Instruct} \\
\cmidrule(lr){2-3}
 & GoV & PRO \\
\midrule
$N=2$ & \textbf{99.40} & 98.37 \\
$N=4$ & 90.59 & \textbf{94.13} \\
$N=6$ & 72.30 & \textbf{79.82} \\
$N=8$ & 46.62 & \textbf{56.54} \\
\bottomrule
\end{tabular}
\label{tab:nts}
\vspace{-1.5em}
\end{wraptable}
We further evaluate PRO on Number Triangle Summation. Unlike open-form and medical reasoning, this task has an explicit dependency structure, while the reasoning depth can be systematically varied through the problem size \(N\). Table~\ref{tab:nts} reports the results across different reasoning depths. When the reasoning chain is short, PRO and GoV exhibit comparable performance, as an introduced error has limited opportunity to propagate before becoming observable. As the reasoning depth increases, PRO shows an increasingly clearer advantage. This trend suggests that the benefit of progression-aware localization becomes more pronounced as the distance between error origin and error exposure grows. The controlled results therefore complement the observations on ProcessBench and MedReason, showing that PRO's advantage is closely associated with error propagation rather than with a particular reasoning domain or semantic structure.

\subsection{Ablation Study}

\begin{wraptable}{r}{0.50\textwidth}
\vspace{-15pt}
\captionof{table}{The ablation study of PRO.}
\vspace{-5pt}
\label{tab:ablation_processbench_qwen32}
\centering
\scriptsize
\setlength{\tabcolsep}{5pt}
\begin{tabular}{ccccc}
\toprule
Backward & Refinement & Top-$k$ Rerank & Err. Acc. & F1 \\
\midrule
\(\times\) & \(\times\) & \(\times\) & 56.28 & 68.64 \\
\(\times\) & \checkmark & \checkmark & 58.94 & 69.28 \\
\checkmark & \(\times\) & \checkmark & 50.72 & 64.51 \\
\checkmark & \checkmark & \(\times\) & 44.93 & 57.78 \\
\checkmark & \checkmark & \checkmark & \textbf{68.12} & \textbf{79.10} \\
\bottomrule
\end{tabular}
\vspace{-15pt}
\end{wraptable}
We further examine the contribution of each component of PRO on ProcessBench using Qwen2.5-32B-Instruct. Table~\ref{tab:ablation_processbench_qwen32} compares the full model with variants of PRO. The results show that all three components contribute to first-error localization. Removing outgoing evidence weakens the ability to distinguish locally plausible blocks from progression-inconsistent ones, while removing refinement makes ambiguous detector outputs harder to resolve. The largest degradation occurs without top-\(k\) reranking, indicating that detector-conditioned source attribution is particularly important for converting exposed failures into accurate error origins. 

\section{Conclusion}
In this paper, we identify a key limitation of sequential reasoning verification. Under error propagation, the first observable failure may not be the step that actually introduces the error. To address this \textit{source--exposure} gap, we propose a training-free framework PRO that models how verification evidence evolves along the reasoning trajectory. PRO combines progression-aware evidence, selective refinement, and source-oriented attribution to trace exposed failures back to their likely origin. Its effectiveness comes from separating exposure detection from source localization and using downstream consequences as retrospective evidence for earlier errors. Experiments across diverse reasoning settings validate this design. More broadly, our results suggest that reliable reasoning verification should model error propagation rather than rely solely on local stepwise judgments. Future work will extend this work to explor longer reasoning traces.

\subsection*{AI use statement}
We used generative AI tools only for limited writing assistance during manuscript preparation, including language polishing, grammar correction, and improvements to stylistic clarity. We did not use generative AI tools for method design, experimental design, result generation, or scientific claim validation. All AI-assisted content was reviewed by the authors. The authors verified all technical content, method design, experimental settings, results, and final wording decisions, and take full responsibility for the final contents of this paper.

\bibliography{library}
\bibliographystyle{paperstyle}

\appendix

\vspace{0.5em}

\setcounter{figure}{0}
\renewcommand{\thefigure}{A\arabic{figure}}

\setcounter{table}{0}
\renewcommand{\thetable}{A\arabic{table}}

\setcounter{equation}{0}
\renewcommand{\theequation}{A\arabic{equation}}

\setcounter{theorem}{0}
\renewcommand{\thetheorem}{\arabic{theorem}}

\section*{Appendix}\label{appendix}
This appendix provides supplementary material for the main paper. Appendix A extends the related-work discussion beyond the compact summary in the main text. Appendix B introduces the implementation-side preliminaries repeatedly used in the supplementary sections. Appendix C gives the theoretical proofs. Appendix D summarizes the main steps of PRO from the detector to the source-localization stage. Appendix E records the experimental setting and reproducibility protocol. Appendix F presents an additional qualitative case study.

\section{Related Work}
\subsection{Reasoning Verification}
Reasoning verification is commonly instantiated either as holistic verification over the full chain or as decomposition-based verification over intermediate units \citep{cobbe2021training,uesato2022solving,fang2025graphverification}. Holistic verifiers provide a single global judgment over the complete trace and remain a strong baseline in many mathematical reasoning settings, especially when the main objective is to determine whether a full solution is acceptable. Their limitations become more visible when the trajectory is long, when multiple local dependencies interact, or when the task requires localized error attribution rather than a single final verdict.

Decomposition-based methods instead expose intermediate units to the verifier, trading simplicity for finer-grained evidence and stronger interpretability. Existing approaches differ in what they treat as the basic unit of verification, ranging from individual steps to adaptive blocks that follow task structure. GoV is representative of the latter direction because it treats verification granularity as a structural choice rather than a fixed design constant \citep{fang2025graphverification}. Recent benchmarks such as ProcessBench and DeltaBench make this distinction especially salient by showing that earliest-error localization remains difficult even for strong process verifiers \citep{zheng2025processbench,he2025deltabench}.

\subsection{Process Supervision}
Process supervision provides learning signals at the level of intermediate reasoning steps rather than only final outcomes \citep{uesato2022solving,lightman2023lets}. This line of work has been influential because it shows that reasoning quality cannot always be recovered from outcome supervision alone: step-level labels often reveal errors that are invisible from final answers. PRM-style supervision and its lower-cost extensions, such as Math-Shepherd and OmegaPRM, further demonstrate that intermediate supervision can be made more scalable through semi-automatic annotation or teacher-assisted labeling \citep{wang2024mathshepherd,luo2024improve}.

The broader literature in this area therefore studies not only whether step-level information matters, but also how such information should be collected, represented, and used. Most existing approaches focus on learned critics, reward models, or annotation pipelines for constructing step-level labels. As a result, process supervision has developed into both a methodological direction for training reasoning systems and an evaluation lens for understanding where intermediate errors arise.

\subsection{Self-Correction and Test-Time Verification}
Test-time verification and self-correction improve reasoning by inserting critique, revision, or verification during inference \citep{madaan2023selfrefine,dhuliawala2024chain,wu2024proco,wan2025derailerrerailer}. These methods mainly target answer improvement or trajectory repair, and they highlight the broader value of inference-time feedback without retraining the base model. Some methods emphasize iterative refinement, while others explicitly separate generation from verification so that the model can revise its own reasoning under external checks.

This literature is also closely connected to recent discussions about the limits of intrinsic self-correction. \citet{huang2023large} show that models often fail to reliably repair their own reasoning without sufficiently informative external signals. This observation has motivated continued interest in external critics, verification-guided refinement, and hybrid test-time pipelines that combine detection, critique, and revision.

\section{Preliminaries}
\subsection{Problem Setting}
This paper studies \emph{first-error localization} for multi-step reasoning traces. Given a problem instance \(q\) and a reasoning trace \(\mathcal{B}=\{b_1,\ldots,b_n\}\), the verifier must return the earliest incorrect block index, or \(-1\) if the full trace is valid. Unlike final-answer verification, this setting requires the verifier to reason about intermediate dependencies inside the trajectory rather than only judging the final output. The same formulation is used across all three datasets in our paper, although the trace structure differs across arithmetic programs, free-form mathematical derivations, and medical reasoning chains.

\textbf{Verification Setting:} The central difficulty is that the first \emph{exposed} inconsistency is not always the true \emph{source} error. A downstream block may be the first place where failure becomes visible, even though the underlying mistake originates in an earlier block. As a result, earliest-error localization cannot be reduced to simply finding the first block that looks suspicious under local checking.

\textbf{Detector and Candidate Space:} To handle this mismatch, the appendix repeatedly distinguishes detector-stage evidence from source-localization evidence. Concretely, the detector returns an index \(j_{\mathrm{det}}\in\{-1,1,\ldots,n\}\), where \(-1\) means that no suspicious block survives block-level screening. The subsequent top-\(k\) stage begins with a proposal list over the full trace, but the final shortlist is detector-conditioned: it always contains \(j_{\mathrm{det}}\) when \(j_{\mathrm{det}}\neq -1\), even if the ranker itself would otherwise omit that block.

\textbf{Implementation Conventions:} We also adopt two conventions that are used throughout the appendix. For the final block \(b_n\), we set \(V_{\mathrm{out}}(b_n)=1\), so that the terminal block is not marked uncertain merely because no successor exists. In addition, we follow a fail-closed protocol: empty responses, malformed verdicts, and API-side failures are not silently promoted to correct judgments, but are treated as uncertain or failed verification outcomes.

\subsection{Goals}
The purpose of PRO is not merely to detect that a reasoning trace fails, but to improve the reliability of earliest-source localization under black-box, training-free verification. To make this objective concrete, the method is designed around the following goals.

\begin{itemize}
\item \textbf{Localization Fidelity.} The verifier should recover the earliest source error rather than only the first downstream block that visibly breaks. This requires the pipeline to distinguish causal origin from propagated inconsistency, rather than treating every exposed downstream failure as equally informative.

\item \textbf{Structural Robustness.} The method should remain effective when trace length, block granularity, and local dependency structure vary substantially across datasets. In other words, the verifier should not depend on a single fixed notion of what a ``natural'' reasoning unit looks like, since arithmetic traces, free-form derivations, and medical reasoning chains expose very different local structures.

\item \textbf{Training-Free Deployability.} The full pipeline should work without additional process-supervision labels, learned reward models, or task-specific verifier training. This constraint is important because the intended setting is black-box verification rather than supervised critic construction, and therefore the method must rely on pipeline design rather than extra training signals.
\end{itemize}

These goals motivate the notation used in the remainder of the appendix: later proofs and analyses repeatedly rely on the distinction between incoming-side signals, outgoing-side mismatch, detector-conditioned candidate construction, and rewrite-supported reranking.

\subsection{Challenges}
The above goals are difficult to satisfy simultaneously because earliest-source localization is not simply a stronger form of ordinary correctness checking. The core difficulty lies in identifying the causal origin of failure rather than merely the first place where failure becomes visible, and doing so under reasoning structures whose local granularity may itself be unstable. In our setting, the main challenges are therefore structural rather than purely operational.

\textbf{Challenge I: Exposure-Source Mismatch.} In many error traces, the earliest block that looks suspicious is not the true origin of failure. This creates a systematic gap between detection and source attribution, because local verification evidence may first become visible only after the original mistake has already propagated. Conventional stepwise verification is therefore biased toward the first exposed downstream inconsistency, even when that inconsistency is merely a consequence of an earlier hidden mistake. The challenge is not only to flag that something has gone wrong, but to separate propagated failure from causal origin without assuming access to gold intermediate labels.

\textbf{Challenge II: Structural Ambiguity Under Variable Granularity.} Even after a suspicious region is exposed, source localization remains difficult because the same local mismatch may admit multiple explanations. A failed bridge can be caused by the current block, by its immediate continuation, or by a coarser inconsistency inherited from an earlier region. Fine-grained verification helps expose local arithmetic or logical faults, but overly small units may discard the context needed for causal attribution; coarser units preserve context, but they blur the precise location of error. A practical method must therefore refine granularity only where the structure becomes ambiguous, while preserving enough context to decide which candidate block is genuinely source-like.

\section{Theoretical Proof}
The theoretical role of PRO is to formalize why forward-only rejection is insufficient for detector-stage structural exposure, and why outgoing-side structure must be incorporated before source localization can be carried out reliably. To support this claim, we prove two complementary statements: the structural-exposure index used by the detector can strictly precede the ordinary forward rejection index, and incoming-side signals alone are insufficient for universally correct recovery of that detector-stage target.

\subsection{Forward Rejection Versus Structural Exposure}
We first compare the earliest index returned by ordinary forward rejection with the structural-exposure index used by PRO. The claim is that the detector-stage index can never occur later than the forward rejection index, and can in fact be strictly earlier when outgoing-side mismatch exposes a suspicious bridge before any incoming-side rejection occurs.
\begin{proof}
\textbf{Candidate-set inclusion.} To show that structural exposure cannot occur later than forward rejection, we compare the two minimized index sets, which can be written as:
\begin{equation}
\tau_{\mathrm{seq}}(q,\mathcal{B})
=
\min\bigl(\{\,i \mid I_i=0\,\}\cup\{\infty\}\bigr),
\end{equation}
and
\begin{equation}
\tau_{\star}(q,\mathcal{B})
=
\min\bigl(\{\,i \mid I_i=0\,\}\cup \mathcal{M}(q,\mathcal{B})\cup\{\infty\}\bigr),
\end{equation}
where \(\mathcal{M}(q,\mathcal{B})=\{\,i<n \mid I_i=1,\ O_i=0\,\}\).
The key set inclusion can be expressed as:
\begin{equation}
\{\,i \mid I_i=0\,\}\cup\{\infty\}
\subseteq
\{\,i \mid I_i=0\,\}\cup \mathcal{M}(q,\mathcal{B})\cup\{\infty\}.
\end{equation}
Therefore, the desired ordering relation can be formulated as:
\begin{equation}
\tau_{\star}(q,\mathcal{B})\le \tau_{\mathrm{seq}}(q,\mathcal{B}).
\end{equation}

\textbf{Strictness witness.} To show that the inequality can be strict, it is enough to construct a trace in which all incoming-side tests pass, but an earlier outgoing-side mismatch exists. Consider any trace satisfying \(I_i=1\) for all \(i\), and suppose there exists some \(j<n\) such that \(O_j=0\). Then \(j\in\mathcal{M}(q,\mathcal{B})\), so \(\tau_{\star}(q,\mathcal{B})\le j<\infty\), while \(\tau_{\mathrm{seq}}(q,\mathcal{B})=\infty\).
Hence the inequality is strict on this instance.
\end{proof}

\subsection{Necessity of Outgoing-Side Information}
We next show that two traces can share identical incoming-side signals while differing in their detector-stage structural-exposure index. This establishes that incoming-side evidence alone is not sufficient for universally correct recovery of that detector target.
\begin{proof}
\textbf{Matched incoming-side signals.} We construct two traces over the same problem instance \(q\) that share identical incoming-side signals but differ in outgoing-side structure. Let \(\mathcal{B}=(b_1,\ldots,b_n)\) and \(\mathcal{B}'=(b'_1,\ldots,b'_n)\) be two decomposed traces such that \(I_k(\mathcal{B})=I_k(\mathcal{B}')=1\) for all \(k\).

\textbf{Different outgoing-side structure.} Assume furthermore that \(\mathcal{B}\) contains an outgoing-side mismatch at the first block, \(O_1(\mathcal{B})=0\), whereas \(\mathcal{B}'\) has no mismatch at all, with \(O_k(\mathcal{B}')=1\) for all \(k<n\). Then \(\mathcal{M}(q,\mathcal{B})\ni 1\), whereas \(\mathcal{M}(q,\mathcal{B}')=\varnothing\).
Since all incoming-side signals are identical and always equal to \(1\), the difference between the two traces must come entirely from outgoing-side structure. The corresponding detector-stage structural-exposure indices can be expressed as:
\begin{equation}
\tau_{\star}(q,\mathcal{B})=1,
\qquad
\tau_{\star}(q,\mathcal{B}')=\infty.
\end{equation}
Therefore, their separation can be formulated as:
\begin{equation}
\tau_{\star}(q,\mathcal{B})\neq \tau_{\star}(q,\mathcal{B}').
\end{equation}

\textbf{Implication for detector rules.} Any detector rule that depends only on the incoming-side sequence \(\{I_k\}\) would assign the same output to \(\mathcal{B}\) and \(\mathcal{B}'\), because those signals are identical by construction. Such a rule cannot be correct on both traces simultaneously. Hence outgoing-side information is necessary for universally correct recovery of \(\tau_{\star}\).
\end{proof}

\section{The Main Steps of PRO}
PRO proceeds in three stages at test time. First, it performs block-level bidirectional checking using incoming admissibility and outgoing-side support against the immediate continuation. When the two signals disagree, PRO treats the case as a structural ambiguity and enters local refinement rather than issuing an immediate rejection. The theoretical pre-refinement exposure index is \(\tau_{\star}\), while the refined detector output used by the implementation is \(\tau_{\mathrm{det}}\).

For every suspicious region that survives detection, PRO uses the refined detector output as an anchor for candidate attribution. It first obtains an ordered proposal list over the full trace, then constructs a detector-conditioned top-\(k\) shortlist by retaining the detector position when necessary before truncating the list to length \(k\). Each shortlisted candidate is locally rewritten and re-verified under the same verified prefix, and the resulting repair evidence is used for source reranking. If reranking does not return a valid candidate from the shortlist, PRO falls back to the original detector index.

Across datasets, PRO uses the same prompt family for forward verification, outgoing-side verification, edge refinement, candidate proposal, candidate rewriting, repair verification, and source reranking; only the verification unit definition changes with the task. For the terminal block \(b_n\), we set \(V_{\mathrm{out}}(b_n)=1\), so that the absence of a successor does not downgrade the final block. Algorithm~\ref{alg:pro-main} summarizes the complete test-time procedure.

\begin{algorithm}[t]
\caption{The main steps of PRO}
\label{alg:pro-main}
\begin{algorithmic}[1]
\Require \(q\), \(\mathcal{B}=(b_1,\ldots,b_n)\), shortlist size \(k\)
\Ensure \(\hat{f}(q,\mathcal{B}) \in \{-1,1,\ldots,n\}\)
\State \(j_{\mathrm{det}} \gets \infty\)
\For{\(i=1\) to \(n\)}
    \State \(I_i \gets V_{\mathrm{in}}(q,\mathcal{P}_{i-1},b_i)\)
    \State \(O_i \gets \mathbf{1}[i=n] + \mathbf{1}[i<n]\cdot V_{\mathrm{out}}(q,b_i,b_{i+1})\)
    \State \(r_i \gets \varnothing\)
    \If{\(I_i=1 \land O_i=0\)}
        \State \(r_i \gets \mathrm{Refine}(q,\mathcal{P}_{i-1},b_i,b_{i+1})\)
    \EndIf
    \If{\(\mathrm{Detect}(I_i,O_i,r_i)=1\)}
        \State \(j_{\mathrm{det}} \gets i\); \textbf{break}
    \EndIf
\EndFor
\If{\(j_{\mathrm{det}}=\infty\)}
    \State \Return \(-1\)
\EndIf
    \State \((c_1,\ldots,c_m) \gets \mathrm{Propose}(q,\mathcal{B})\)
    \State \((c_1^\star,\ldots,c_{m^\star}^\star) \gets \mathrm{Shortlist}(q,\mathcal{B},(c_1,\ldots,c_m),k,j_{\mathrm{det}})\)
\For{\(r=1\) to \(m^\star\)}
    \State \(j \gets c_r^\star\)
    \State \(\tilde{b}_r \gets \mathrm{Rewrite}(q,\mathcal{P}_{j-1},b_j)\)
    \State \(e_r \gets \mathrm{VerifyRepair}(q,\mathcal{B},j,\mathcal{P}_{j-1},\tilde{b}_r)\)
\EndFor
\State \(j^\star \gets \mathrm{Rerank}(q,(c_1^\star,\ldots,c_{m^\star}^\star),\{(r,e_r)\}_{r=1}^{m^\star},j_{\mathrm{det}})\)
\If{\(j^\star \in \{c_1^\star,\ldots,c_{m^\star}^\star\}\)}
    \State \Return \(j^\star\)
\Else
    \State \Return \(j_{\mathrm{det}}\)
\EndIf
\end{algorithmic}
\end{algorithm}

\section{Experimental Setting}
\label{sec:appendix-exp-setting}

\subsection{Inference and Data Configuration}
All compared methods use the same backbone within each setting. The main evaluation runs use deterministic decoding at temperature \(0\), while any retry behavior follows the released task-specific scripts. PRO uses a detector-conditioned shortlist of \(k=5\). ProcessBench is evaluated on its complete GSM8K, MATH, and OlympiadBench test subsets. The MedReason evaluation set contains 450 traces, comprising 225 unmodified traces and 225 corrupted counterparts balanced across nine source partitions. The injected error position is used as the first-error label, and the same evaluation set is used for every compared backbone.

\subsection{Metrics}
\label{sec:appendix-metrics}
In this paper, F1 always denotes exact first-error localization F1. When dataset annotations permit, we additionally report Correct Accuracy and Error Accuracy. These definitions are used throughout the main text and appendix because they directly reflect whether the verifier recovers the true source step rather than merely exposing that some downstream failure exists.

\textbullet\ \textbf{Correct Accuracy:} The proportion of fully correct reasoning traces that are correctly judged as error-free.

\textbullet\ \textbf{Error Accuracy:} The proportion of erroneous reasoning traces for which the method successfully identifies the first incorrect step.

\textbullet\ \textbf{F1:} The harmonic mean of Correct Accuracy and Error Accuracy, used as the primary metric for overall first-error localization performance.

\section{Qualitative Case Study}
To understand how the gain appears at the trajectory level, we inspect a representative MedReason example in which GoV fails to recover the annotated source step, while PRO localizes it correctly after candidate-conditioned rewriting and reranking.

\begin{figure}[t]
\centering
\includegraphics[width=0.72\textwidth]{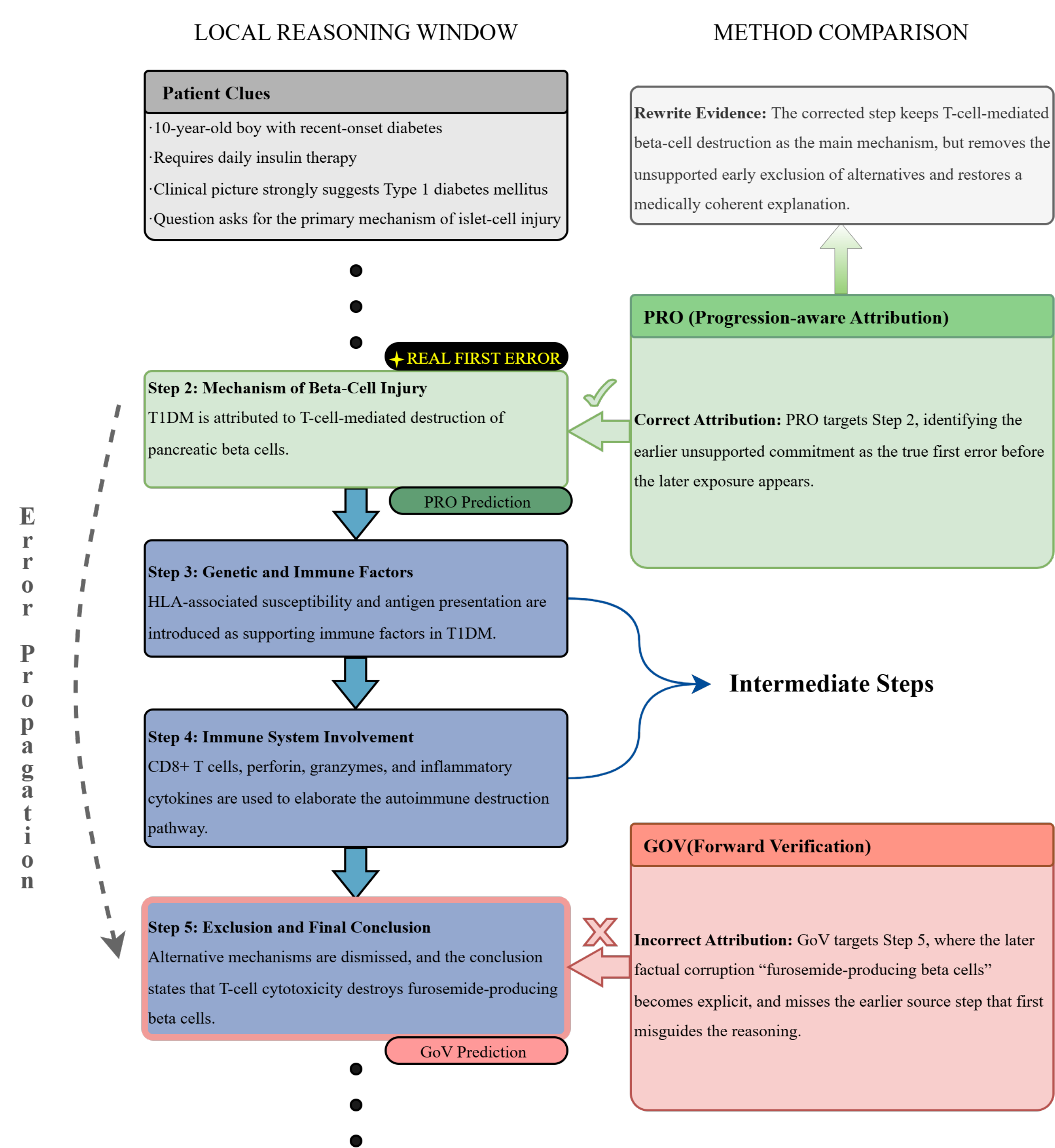}
\caption{Representative MedReason case study. Type 1 diabetes reasoning trajectory. The true first error is at Step 2. GoV misattributes the failure to Step 5 where corruption becomes explicit, while PRO correctly identifies the early source error.}
\label{fig:medreason_case}
\end{figure}

The selected example asks for the primary mechanism of islet-cell injury in a 10-year-old boy with recent-onset diabetes requiring daily insulin therapy, whose clinical picture strongly suggests Type 1 diabetes mellitus. In the corrupted trajectory, the downstream intermediate steps (Steps 3 and 4) remain locally plausible: they correctly introduce HLA susceptibility, antigen presentation, CD8+ T cells and inflammatory cytokines as supporting immune factors for T1DM. The actual corruption appears early in Step 2, in the form of an unsupported early commitment to a causal claim. This error is subtle and propagates through subsequent reasoning; downstream steps remain coherent until Step 5, where the factual mistake (“furosemide-producing beta cells”) becomes explicit.

GoV does not isolate this source step reliably, since the intermediate statements about genetic susceptibility and immune mediators still look acceptable under the corrupted prefix. PRO instead revisits the suspicious region with candidate-conditioned rewriting, tests whether a revised version restores consistency with both the preceding clinical context and the final conclusion, and then reranks the surviving candidates by source plausibility. In this case, PRO identifies Step 2 as the first erroneous block and recovers the true source location. This example illustrates why progression-aware attribution helps in medically constrained traces: locally plausible intermediate reasoning steps can still be built upon an earlier unjustified claim, and correcting that initial flawed assembly point matters more than simply flagging the first downstream factual mismatch.

\end{document}